\documentclass[runningheads]{llncs}
\usepackage[T1]{fontenc}
\usepackage{graphicx}
\DeclareSymbolFont{symbolsC}{U}{txsyc}{m}{n}
\DeclareMathSymbol{\strictif}{\mathrel}{symbolsC}{74}
\usepackage{multicol}
\usepackage{tabularx}
\usepackage{ragged2e}
\usepackage{listings}
\usepackage{graphicx}   
\usepackage{xcolor}  
\usepackage{hyperref}
\usepackage{url}         
\usepackage{multirow}     
\usepackage{amssymb}     
\usepackage{amsmath}      
\usepackage{yfonts}
\usepackage{lipsum}      
\usepackage{proof}
\usepackage{amssymb}
\usepackage{mathrsfs}
\usepackage{mathtools}
\usepackage[tableau]{prooftrees}

\newcommand{\sent}{\rightarrow}

\newtheorem{notation}{Notation}

\usepackage{color}

\begin{document}
\title{Study on Quantitative Dynamic Epistemic Logic for Belief Revision}
%
%
\author{Felipe Camargo\inst{1}\orcidID{0009-0002-8336-3669}}
\authorrunning{F. Camargo}
%
\institute{State University of Campinas (UNICAMP), Campinas-SP, Brasil \email{f247210@dac.unicamp.br}}
\maketitle              
\begin{abstract}
Belief revision is a process in which an agent begins to believe in something she previously did not. I begin the paper by presenting, based on \cite{gar,han}, postulates for belief revision that constitute the basis of the AGM theory. I will then briefly show the semantics of a modal logic introduced in \cite{prole}, which I call `$P$'. This logic formalizes static epistemic states and has greater expressive power than AGM in doing so because it captures the quantitative notion of "degrees of conviction". The third step is to introduce revision operators on $P$ and, mostly following \cite{prole}, obtain the Dynamic Epistemic Logic (DEL) I call `$P*$'. It models processes of belief revision in several ways. Original results are presented in the following two sections. The first one of these sections revolves around a formalization of AGM postulates within $P*$ by proving some theorems related to the satisfaction of those postulates by revisions defined in $P*$. The last section features an analysis of $P*$'s revisions that go beyond the mere satisfaction of postulates. I compare their formal behavior with respect to some philosophical \textit{criteria}. At last, I conclude that the functions presented in \cite{prole} are not good formalizations of the philosophical intuition behind AGM. Instead, it is captured by the function $*^0$ originally defined in this paper (but highly inspired by \cite{van}). An implementation of this function is also provided.

\end{abstract}
\section{Introduction}

Belief revision is a process in which an epistemic agent (e.g., people, databases, institutions, or A.I.) comes to believe something new. This process is triggered by changes  in the environment or the acquisition of knowledge by the agent. It is usually assumed that the revised belief (that is, the one that comes to be believed) was not already accepted by the agent before revision. Otherwise, there would be no material change in her epistemic state. After this process, the agent accepts a new belief, which is said to be the object of that revision. In order to avoid inconsistencies when accepting the new belief, she may need to give up some of her prior beliefs. Belief revision is, therefore, a type of belief dynamics, as it changes agent's epistemic state. 
 
 The first explicit studies of this process followed the so-called AGM paradigm, named after its founders. AGM traditionally deals with one-agent scenarios, that is, situations in which no more than one agent is considered a "belief revision realizer". All other agents, if present, are considered mere elements of the world. Indeed, another aspect of AGM is separating the world from the agent, in the sense that the theory only deals with the agent's beliefs related to the world (called `first-order beliefs'), never the introspective beliefs related to epistemic states themselves (higher-order).\footnote{On multi-agent scenarios, not all higher-order beliefs are introspective (they may be about the states of other agents). For this reason, in most cases, formalizations of multi-agent scenarios must include higher-order beliefs.} Consequently, no beliefs about the epistemic state of the agent can be revised (i.e., can be treated as an object of revision).
 
 On the other hand, the DEL paradigm, despite also dealing with belief dynamics, is usually focused on multi-agent scenarios that include higher-order intersubjective beliefs. The employed tools usually come from modal logics with several modal operators, one for each epistemic attitude that each agent may hold. Despite dealing with DEL, this paper does not address multi-agent scenarios. The intended use of the DEL paradigm in this paper is to manipulate the \textit{quantitative} aspect of belief that both traditional AGM and some DELs (such as the one in \cite{qual}) do not treat as a central concern. With a quantitative account of belief, one can formally capture the notion of "degree of conviction" and use it as the main tool to represent belief revisions. For a more detailed discussion about the introduction of quantitative aspects into DELs, consult \cite{qual}.\footnote{The position of the authors of that work goes against the introduction of such aspects on DELs' account of belief dynamics. They believe that quantitative elements should be added \textit{a posteriori}. Still, I consider that the earlier introduction of such elements as fundamental aspects of the logics allow for a more intuitive and simple account of degrees of conviction.}  

 The DEL presented here, which models the intuitions behind the AGM postulates, was introduced in \cite{prole} and is called `$P*$' (by me). This system provides a plurality of dynamic functions that change the epistemic state of a given agent. The main contributions of this paper are found in Sections \ref{SAGMP*}-\ref{SAGMP*+}, in which I evaluate, through formal postulates, how well a function of $P*$ captures the intuitions behind AGM postulates. I also argue why the function $*^0$ introduced here is better than each of the others presented in \cite{prole} for capturing these intuitions. An implementation of the function is also presented, providing a practical way to apply and deal with it. Thus, I hope to enlarge the logico-philosophical understanding of the behavior that quantitative elements of belief manifest during revisions.   

\section{Brief Overview on AGM}\label{SAGM}
    
     C. Alchourrón, P. G\"ardenfors and D. Mackinson proposed in \cite{agm} several postulates establishing rules for belief revision. These have been better described in \cite{gar,han}. The postulates are related to a binary function $*: (T,\alpha)\longmapsto T^*_\alpha$, where $T$ is a consistent theory of some consequence system $L= (F_L, \vdash_L)$, \footnote{This means $T\subseteq F_L$ is a nontrivial set of formulas, that is, $T\neq F_L$ and $T$ is closed under the consequence relation $\vdash_L$, i.e., $T\vdash_L\alpha\implies\alpha\in T$.} $\alpha\in F_L$,\footnote{i. e., $\alpha$ is a formula of the language of $L$.} and $T^*_\alpha\subseteq F_L$. In principle, the postulates would make sense for several different consequence systems $L$, but it is still necessary for $L$ to be based on classical propositional logic. If this is not the case, the postulates presented in Definition \ref{AGM*} would require significant adaptations. For this reason, it is hereafter tacitly assumed that $L$ is a propositional logic containing classical propositional logic. This means that $L$ can express the classical connectives $\neg$ and $\land$, based on which the other classical connectives can be defined.      

      Let the propositional variables of $F_L$ be $p_i$ for $i\in \mathbb{N}$. For each $\Gamma\subseteq F_L$, take $Cn_L(\Gamma)$ to be $\{\alpha\in F_L:\Gamma\vdash_L \alpha\}$ (the set of $L$-consequences of $\Gamma$). As expected, when $\{\alpha,\beta\}\subseteq F_L$, the relation $\alpha\equiv_L\beta$ states that $\alpha\vdash_L\beta$ and $\beta\vdash_L\alpha$. When the context makes the consequence system $L$ obvious, I will omit the subscript of $F_L$, $Cn_L$, $\vdash_L$ and $\equiv_L$. In this scenario, formulas are meant to represent beliefs, theories formalize epistemic states of agents by representing the set of all their beliefs, and functions $*: (T,\alpha)\longmapsto T^*_\alpha$ formalizes belief revision. $L$ states which beliefs are incompatible with each other and which are inseparable. For instance, one could say that it is impossible for an agent to believe in two formulas but not in their conjunction. This would be captured by the fact that $\{\alpha,\beta\}\vdash_L\alpha\land\beta$.     

     An input of $*$ is understood as the epistemic state of an agent along with the object of revision. The output, on the other hand, is the epistemic state of the agent after the revision by the input belief. The basic postulates that rule this process are as follows:\footnote{They are called `basic' because more postulates could be added, in order to establish more rationality \textit{criteria}.}

\begin{definition}[AGM Revision]\label{AGM*}
    $*$ is an AGM-like belief revision function (over $L=(F,\vdash)$) if, and only if, for each consistent theory $T\subseteq F$ and any $\alpha, \beta\in F$:
    \begin{equation}\label{K*1}
        Cn(T^{*}_{\alpha})= T^{*}_{\alpha}
    \end{equation} 
    \begin{equation}\label{K*2}
    \alpha \in T^{*}_{\alpha}    
    \end{equation} 
    \begin{equation}\label{K*3}
        T^{*}_{\alpha}\subseteq Cn(T\cup \{\alpha\})
    \end{equation} 
    \begin{equation}\label{K*4}
        \neg \alpha \notin T \implies Cn(T\cup \{\alpha\})\subseteq T^{*}_{\alpha}
    \end{equation} 
    \begin{equation}\label{K*5}
        \vdash \neg \alpha \iff T^{*}_{\alpha} = F
    \end{equation} 
    \begin{equation}\label{K*6}
        \alpha \equiv \beta \implies T^{*}_{\alpha}=T^{*}_{\beta}
    \end{equation} 
    \begin{equation}\label{K*7}
        T^{*}_{\alpha\land \beta} \subseteq Cn(T^{*}_{\alpha}\cup \{\beta\})
    \end{equation} 
    \begin{equation}\label{K*8}
        \neg\beta\notin T^{*}_{\alpha} \implies Cn(T^{*}_{\alpha}\cup \{\beta\})\subseteq T^{*}_{\alpha\land \beta}
    \end{equation} 
\end{definition}

   Before describing the informal meaning of each postulate, it is important to state that AGM formalizes \textit{rational} belief revision. Therefore, there may be several counterexamples to the postulates in real-world scenarios. However, the existence of such scenarios does not invalidate the framework. This simply means that revisions on those counterexamples are irrational. 
    
        Postulate \ref{K*1} states that $*$ always outputs a theory. Postulate \ref{K*2} (often called the `Postulate of success'; revisions that validate it are called `successful') states that this theory succeeds in accepting the revised formula. After revision, the revised formula is believed. Postulate \ref{K*3} establishes that, after executing a revision by a formula, no agent believes in anything more than the beliefs she used to have along with the new one (and everything that follows from these). Postulate \ref{K*4} states that when the agent did not previously believe in the negation of what she now wants to accept, she comes to believe in everything that Postulate \ref{K*3} allows her to. Otherwise, she must give up some beliefs to accept the new one. Postulate \ref{K*5} states that one must only trivialize beliefs after accepting that a logical falsity is true. Hence, no matter how incompatible the new belief is with the old ones, trivialization will not occur as long as it is logically possible for the new belief to be true. Postulate \ref{K*6} states that syntax is irrelevant to revision. The same belief, even if stated in two different ways, generates the same $*$-output. Postulates \ref{K*7}-\ref{K*8} are generalizations of Postulates \ref{K*3}-\ref{K*4} relating the revision of a conjunction to iterated revision. In fact, since (by Postulates \ref{K*1} and \ref{K*4}) $T^*_\top=T$, Postulates \ref{K*3}-\ref{K*4} are instances of Postulates \ref{K*7}-\ref{K*8} for $\alpha=\top$.

     Finally, one should notice what is exactly the process formalized by those postulates and by the theorems they entail. The idea is that the revised theories $T^*_\alpha$ are the formulas in which the agent beliefs after receiving the information $\alpha$, but, among those formulas, $T^*_\alpha$ includes only the ones about what used to be true in the original state, before receiving the information. In other words, AGM does not capture any of the information that has come to be true \textit{in virtue of the revision procedure itself}. For instance, if $\alpha=p_1$, the revised sets $T^*_\alpha$ must include statements such as `$p_1$', `$p_1\lor p_2$' etc., but not statements corresponding to `it was informed that $p_1$'.\footnote{Consult \cite{qual} for more about this.} Furthermore, the only beliefs took into account are the ones encoded by formulas in $F_L$, and this set does not include expressions such as `$\beta\in T^*_\alpha$'. This means that the agent is never informed about her epistemic state. Traditional AGM do not deal with such information. The only formulas to be formally revised and believed are those related to first-order beliefs. For a detailed account of AGM, consult \cite{han}.

\section{Modal Logic With Preference Functions}\label{SP}

The AGM does not focus on the quantitative aspects of belief. Two people may hold the same beliefs but still revise them differently, depending on which beliefs are preferred (or held more firmly) by one or the other. As a rule, AGM deals with this by addressing an epistemic entrenchment order that dictates which formulas are epistemically preferred. Still, the order only has an instrumental role and operations on it are typically not addressed in detail. Epistemic entrenchment serves qualitative revisions in a minimal way in order to make them behave as expected.

On the other hand, \cite{prole} presents a prolific way to study belief revision functions in a modal setting that encodes the quantitative aspects of beliefs. Furthermore, it allows for a formalization (defined in Section \ref{SAGMP*}) of an adapted version of AGM that deals with introspective beliefs. This is done with a logic I am calling `$P$', defined by a class of models with preference functions (also called in the literature "plausibility models", "sphere-based models", etc.). The present description of $P$ differs from \cite{prole} mainly in scope, since not all information and tools provided by van Ditmarsch (such as those related to multi-agent scenarios) are actually relevant to the way the logic defines AGM-like belief revision functions.
 
 The set of formulas $F_P$ is defined by expanding the propositional signature of the previous language $F_L$ with the following unary connectives: $\square^\mathbb N$ as well as one connective $\square^x$ for each $x\in \mathbb N$. The intended meaning of each modal operator is as follows: `$\alpha$ is believed with certainty' for $\square^\mathbb{N}\alpha$ and `$\alpha$ is believed with a conviction degree $x$' for each $\square^x\alpha$. When $x=0$, $\square^x\alpha$ simply means that $\alpha$ is believed. Semantics is characterized by the following class of models:

 \begin{definition}\label{DEF PMODEL}
     A $P$-model is a triple $\mathcal{M}=(W,\pi,V)$ such that:
     \begin{enumerate}
         \item $W\neq\varnothing$ (called a set of worlds).
         \item $V:\{p_i: i\in\mathbb{N}\}\longrightarrow \mathscr{P}(W)$ (called a modal valuation).
         \item $\pi :Plaus\longrightarrow \mathbb{N}$ is such that $Plaus\subseteq W$ (called a preference function).
     \end{enumerate}
 \end{definition}

Local satisfiability ("being true in a world of a model") is recursively defined by:\footnote{When $\mathcal M$ is obvious in $\mathcal M,w\vDash\alpha$, it will be omitted.}

\begin{definition}\label{PSAT}
    $\mathcal{M}, w \vDash \alpha$ ($\alpha$ is true at $w$ in $\mathcal{M}$) iff:
    \begin{enumerate}
        \item $w \in V(\alpha)$ (for $\alpha\in\{p_i: i\in\mathbb{N}\}$).
        \item $\mathcal{M}, w\vDash \alpha_1$ and $\mathcal{M}, w\vDash \alpha_2$ (for  $\alpha=\alpha_1\land\alpha_2$).
        \item $\mathcal{M}, w\nvDash \alpha_1$(for $\alpha = \neg\alpha_1$).
        \item $\mathcal{M}, w'\vDash \alpha_1$ for all $w'$ such that $\pi(w')\leq x$ (for $\alpha=\square^x\alpha_1$).
        \item $\mathcal{M}, w'\vDash \alpha_1$ for all $w'$ such that $w'\in Plaus$ (for $\alpha=\square^\mathbb{N}\alpha_1$).
    \end{enumerate}

    \end{definition}

    There may be more items in the definition above, since the consequence system $L$ could be quite complex. However, as the consequence system is based on classical logic, the definition must include at least the mentioned items. 

    I also establish the following usual notation:
    
\begin{notation}\label{DEF1}
    For any $\Gamma,\{\alpha\}\subseteq F_P$, $\Gamma\vDash_P\alpha$ iff, for all worlds $w$ of any $P$-model $\mathcal{M}$, $\mathcal M,w\vDash\gamma$ for all $\gamma\in \Gamma$ implies $\mathcal{M},w\vDash \alpha$. When $\Gamma=\{\gamma\}$, I use $\gamma\vDash_P\alpha$ instead of $\{\gamma\}\vDash_P\alpha$. In this way, $P$ can be seen as the consequence system $(F_P,\vDash_P)$. The function $Cn_P(\cdot)$ and the relation $\equiv_P$ are defined in a way analogous to $Cn_L(\cdot)$ and $\equiv_L$.
\end{notation}

    Interpreting this system is somewhat involved. A model represents a static epistemic state of an agent. The function $\pi$ informs how plausible each world is. However, \textit{the greater $\pi(w)$ is, the less plausible $w$ is in that epistemic state} ($\pi$ may be better described as an implausibility function in this sense). For example, $\mathcal{M}, w\vDash \square^2\alpha$  means that, in $w$, $\alpha$ is true in all worlds agent $\mathcal{M}$ assigns an (im)plausibility degree of two or less. Notice that, because the $4^{th}$ clause of Definition \ref{PSAT} does not depend on $w$, $\mathcal{M}, w\vDash \square^2\alpha$ implies $\mathcal{M}\vDash \square^2\alpha$.\footnote{The notation $\mathcal{M}\vDash \alpha$ means that $\alpha$ is true in all worlds of $\mathcal{M}$.} In fact, this holds for all statements of the forms $\square^x\alpha$ or $\square^\mathbb{N}\alpha$. Furthermore, $\mathcal{M}, w\vDash \square^2\alpha$ implies $\mathcal{M}, w\vDash \square^1\alpha$ and $\mathcal{M}, w\vDash \square^0\alpha$. This is easy to check for all $x\in \mathbb{N}$ (not only 2). Therefore, the bigger $x$ is in a statement like $\mathcal{M}, w\vDash \square^x\alpha$, the higher the confidence agent $\mathcal{M}$ has in $\alpha$ is. \textit{One confidently believes in something when it holds even in very implausible worlds.}
    
    It follows naturally that $\square^\mathbb{N}$ is a knowledge operator, such that $\square^\mathbb N \alpha$ means that the agent is completely sure that $\alpha$ is true. For statements regarding worlds in Plaus, $\square^\mathbb{N}$ can, in fact, be seen as a Kripke-like operator associated with an equivalence relation that partitions $Plaus$ from $W{-}Plaus$. Consequently, $\square^\mathbb{N}$ works as an $S5$ operator when evaluated on worlds in $Plaus$. The fact that $S5$ is a standard formalization of knowledge guarantees that this semantics aligns with the intended intuition. Nevertheless, a somewhat undesirable result occurs in implausible worlds. On them, it is possible that $\mathcal{M},w\nvDash\square^\mathbb{N}\alpha\sent \alpha$. Since $\square^\mathbb{N}\alpha\sent \alpha$ holds only in plausible worlds, when using $P$ to model phenomena, one should always take into account that each world could, in principle, be in $W-Plaus$. Some restrictions on the modeling process may exclude those situations, but it is important to know that the restrictions must be introduced \textit{a posteriori} and are not a built-in feature of $P$. On the other hand, the $\square^0$ operator is considered a belief operator, since it is the weakest form of "epistemically accepting" a formula. 
    
    The fact that $\mathcal{M}, w\vDash \square^x\alpha \implies \mathcal{M}\vDash \square^x\alpha$ guarantees full positive and negative introspection. That is, $\mathcal{M}, w\vDash \square^x\alpha \implies \mathcal{M},w\vDash \square^\mathbb N\square^x\alpha$ and $\mathcal{M}, w\vDash \neg\square^x\alpha \implies \mathcal{M},w\vDash \square^\mathbb N\neg\square^x\alpha$. This is the case because (negations of) modal statements, when true, are true in all worlds. In particular, in all plausible worlds, but being true in all plausible worlds is being knowingly true. Finally, it is important to say that when $\pi(w)=x$, the intended interpretation is that the agent cannot distinguish $w$ from any other world with $\pi$-degree of $x$. Consequently, the $0$-degree worlds are exactly the worlds the agent considers to be indistinguishable from the actual.

    \begin{example}\label{EX1}
        Consider only the two propositional variables $p_0$ and $p_1$. Now take the model $\mathcal{M}$ in which $w_1\vDash p_0\land p_1$, $w_2\vDash p_0\land \neg p_1$, $w_3\vDash \neg p_0\land p_1$ and $w_4\vDash \neg p_0\land \neg p_1$. The $\pi$ function is as follows: $\pi(w_1)=0$, $\pi(w_2)=1$ and $\pi(w_3)=2$. Notice that $w_4\notin Plaus$. This model is depicted in Figure \ref{fig1}, where the smallest rectangle shows all worlds $w$ such that $\pi(w)\leq 0$, the dashed rectangle shows the ones with $\pi(w)\leq 1$, and the dotted one the worlds with $\pi(w)\leq 2$ (in this case, $\pi(w)\leq 2\iff w\in Plaus$).

\begin{figure}\label{fig1}
\includegraphics[width=0.6\textwidth]{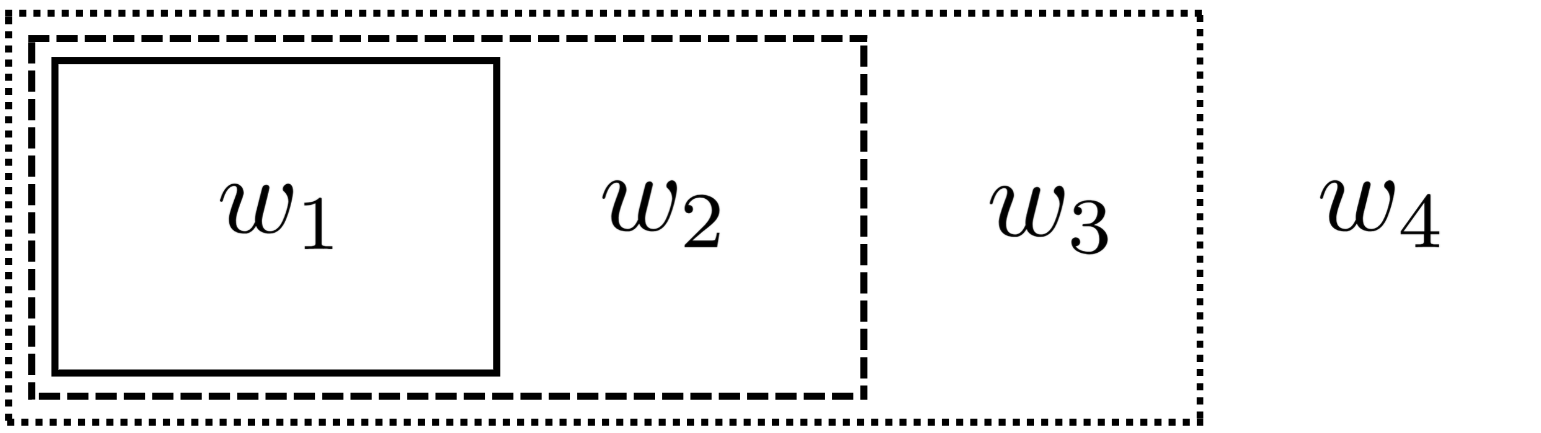}
\centering
\caption{The model $\mathcal{M}$ of Example \ref{EX1}}
\centering
\end{figure}
    
    Delimiting the areas in this way makes the process of deciding whether $\mathcal{M}\vDash \square^x\alpha$ a matter of checking whether $\mathcal{M},w_n\vDash \alpha$ for each world $w_n$ inside the shape associated with the $x^{th}$ degree. In this case, the agent knows that $p_0\lor p_1$, believes that $p_0\land p_1$, but between $p_0$ and $p_1$, she believes in $p_0$ with greater conviction ($\mathcal M\vDash\square^1p_0$, but $\mathcal M\nvDash\square^1p_1$).
    \end{example}

    $P$, then, provides a simple and intuitive way of dealing with the quantitative aspect of a belief, encoding the degree in which a formula is believed. But, as you can see, $P$ is not a DEL, since it encodes only static epistemic states. Still, as shown below, revisions can be defined over the models of $P$. Those revisions are, in fact, simply alterations on $\pi$. This puts the quantitative aspect of belief in the center of the revision process.

\section{Dynamic Modal Logic With Preference Function}\label{SP*}

    It is now possible to build, upon $P$, a $DEL$ I will call `$P*$'. The language $F_{P*}$ is obtained by introducing binary dynamic belief revision operators in $F_P$. The operators are the following six: $[*^i\alpha]\beta$ for $i\in\{0,1,2,3,4,5\}$. Each formula of the form `$[*^i\alpha]\beta$' reads as "after $i$-revising by $\alpha$, $\beta$". I will frequently consider the formula $\alpha$ as fixed and the connective $[*^i\alpha]$ as a unary modal connective.   
    
    The basis for the semantics of each $[*^i\alpha]$ is that \[\mathcal{M},w\vDash[*^i\alpha]\beta \iff\mathcal{M}_\alpha^i,w\vDash\beta\] where $\mathcal{M}^i_\alpha=(W, \pi^i_\alpha, V)$. $W$ and $V$ are exactly as in $\mathcal{M}$. This means that each $[*^i\alpha]$ operator corresponds to some binary $*^i$ function such that \[*^i:(\mathcal{M},\alpha)\longmapsto\mathcal{M}_\alpha^i=(W, \pi^i_\alpha, V)\]     
    
    Defining each operator $[*^i\alpha]\beta$ is now reduced to defining the correspondent $*^i$. This process begins by establishing, for arbitrary $\alpha\in F_{P*}$ and $\pi$ in $\mathcal{M}$, the partially revised function $\overline{\pi}^i_\alpha$, and, based on it, building the revised function $\pi^i_\alpha$ by a process called \textit{normalization}. I will now present the procedure for building the revised preference functions $\pi^i_\alpha$ (through their partially revised counterparts):

\begin{definition}\label{*}
    \begin{equation}\label{*1}
        \overline{\pi}^1_\alpha(w)=
        \begin{cases}
        \pi(w)& \text{if $\mathcal{M}, w\vDash \alpha$}\\
        \pi(w)+ 1 & \text{otherwise}
        \end{cases}
    \end{equation}

    \begin{equation}\label{*2}
        \overline{\pi}^2_\alpha(w)={\pi}(w)\;\; \text{if $\mathcal{M}, w\vDash \alpha$ (and }
        w\notin Dom(\overline{\pi}^2_\alpha)\; \; \text{otherwise)}
    \end{equation}

    \begin{equation}\label{*3}
        \overline{\pi}^3_\alpha(w)=
        \begin{cases}
        \pi(w)& \text{if $\mathcal{M}, w\vDash \alpha$}\\
        \max\{\pi(w), 1\} & \text{otherwise}
        \end{cases}
    \end{equation}

    \begin{equation}\label{*4}
        \overline{\pi}^4_\alpha(w)=
        \begin{cases}
        \pi(w)& \text{if $\mathcal{M}, w\vDash \alpha$}\\
        \max\{\pi(w') : \mathcal{M}, w'\vDash \alpha\} + 1 & \text{otherwise}
        \end{cases}
    \end{equation}

    \begin{equation}\label{*5}
        \overline{\pi}^5_\alpha(w)=
        \begin{cases}
        \pi(w)- \min\{\pi(w'): \mathcal{M}, w'\vDash \alpha\}& \text{if $\mathcal{M}, w\vDash \alpha$}\\
        \pi(w) - \min\{\pi(w'): \mathcal{M}, w'\vDash \neg\alpha\}+1 & \text{otherwise}
        \end{cases}
    \end{equation}

    \begin{equation}\label{*0}
        \overline{\pi}^0_\alpha(w)=
        \begin{cases}
        \pi(w)& \text{if $\mathcal{M}, w\vDash \alpha$}\\
        \pi(w) + \max\{\pi(w') : \mathcal{M}, w'\vDash \alpha\} +1& \text{otherwise}
        \end{cases}
    \end{equation}

    Normalization assigns a new value to each $\pi^i_\alpha(w)$ based on $\overline{\pi}^i_\alpha$. The assignment is:
    \begin{equation}\label{normalization}
        \pi^i_\alpha(w) =\overline{\pi}^i_\alpha(w) - \min\{\overline{\pi}^i_\alpha(w') : w'\in W\}
    \end{equation}
\end{definition}

    This procedure, from $\overline{\pi}^i_\alpha$ to $\pi^i_\alpha$, is the same for any $*^i$. These functions differ only in the first (partial) step. Notice that it is not the case, in general, that $Dom(\pi^i_\alpha)=Plaus$, nor $Dom(\overline \pi^i_\alpha)=Plaus$. An exception to this rule is $\overline{\pi}^2_\alpha$. Still, $\mathcal M^2_\alpha$ is a $P$-model, since its preference function ${\pi}^2_\alpha$ is still a function from a subset of $W$ into $\mathbb N$ (check Definition \ref{DEF PMODEL}). It is important to point out that functions $*^1-*^5$ were proposed by van Ditmarsch in \cite{prole}. $*^0$ is being introduced by me based on an operator van Benthem defined in \cite{van} for another formal system under the name `lexicographic upgrade'.

    The purpose of each $*^i$ is to decide which epistemic state will an agent modeled by $\mathcal{M}$ have after receiving evidence that $\alpha$. This epistemic state is $\mathcal{M}^i_\alpha$. It is now clear why $W$ and $V$ remain unchanged: receiving evidence does not change the state of affairs. It only changes the epistemic preferences of the agent. Notice that this notion follows the usual AGM posture of separating the agent from the world and ignoring how the epistemic deeds, on themselves, are alterations of the world. This interpretation also explains why most of the partial revision processes increase the degrees of worlds that do not validate $\alpha$ (and, in the case of $*^2$, even eliminate them from $Plaus$) while keeping the degree of the other ones unchanged. This is because, when receiving evidence that $\alpha$, the situations that do not validate $\alpha$  should become less plausible relative to the ones that do. 
    
    Normalization guarantees that the final function $\pi^i_\alpha$ is such that $0\in Im(\pi^i_\alpha)$\footnote{Where $Im(f)=\{x: x=f(y)\text{ for some }y\in Dom(f)\}$ denotes the image of a function $f$.} by shifting all the $\pi^i_\alpha$-values towards $0$. Otherwise, several revisions would output agents with trivial beliefs, that is, models $\mathcal M^i_\alpha$ such that $\mathcal M^i_\alpha\vDash \square^0\alpha$ for all $\alpha\in F_{P*}$. Notice that normalization either keeps the world's degrees unchanged or reduces them. This means that, when the partial step of a function does not change the degree of a world, the function as a whole (when the function $\pi$ is directly compared to $\pi^i_\alpha$, without considering the partial step) may still reduce it. On the other hand, when the function $\pi$ is directly compared to $\pi^i_\alpha$, it may be the case that some worlds with degree unchanged had, in fact, the degree increased in the partial step and decreased back to the original value. 

    \begin{example} Consider the model $\mathcal{M}$ of Figure \ref{fig1}. $\overline{\pi}^0_{\neg p_0}$ is as follows:
    $\overline{\pi}^0_{\neg p_0}(w_3)=2$ (since $\mathcal{M}, w_3\vDash \neg p_0$), but $\overline{\pi}^0_{\neg p_0}(w_1)=3$ and $\overline{\pi}^0_{\neg p_0}(w_2)=4$ (by Definition \ref{*}(\ref{*0})). Definition \ref{*}(\ref{normalization}) establishes, then, that $\pi^0_\alpha(w_1)=1$, $\pi^0_\alpha(w_2)=2$ and $\pi^0_\alpha(w_3)=0$. $*^0$ separates the worlds into two groups: those that validate the revised formula and those that do not. Then, it assigns the lowest values to the first group and higher values to the second one. But the function preserves the internal order of both groups. In this concrete case, after receiving evidence that $\neg p_0$, the agent believes in $\neg p_0$, but also in $p_1$, since she already believed in it and nothing about $p_1$ was learned. Indeed, her conviction in $p_1$ is greater than in $\neg p_0$. Her knowledge remains unchanged.
    \end{example}

    Here follows a brief explanation of the other "revision policies" presented. $*^1$ captures cases in which the agent receives a mild evidence for a formula, making all situations in which it is false a bit more implausible. Consequently, it makes all situations in which the formula is true relatively more reasonable. $*^2$, on the other hand, makes the revised formula completely and eternally certain. This policy may formalize situations in which an agent has contact with a formal proof of the revised formula. $*^3$ makes all the situations that do not validate the revised formula as plausible as they can be while still allowing the agent to accept the revised formula.\footnote{"Accepting" here means that the formula is true in all worlds with degree $0$ (i.e., believed).} This revision policy can be seen as a minimal way of making the revision successful. $*^4$ is similar to $*^0$, but instead of preserving the order between the worlds that do not validate the revised formula, the agent simply considers all of them equally implausible. Finally, $*^5$, as $*^0$, preserves the internal order between the $\alpha$-worlds and the $\neg\alpha$-worlds, but instead of assigning to the $\neg\alpha$-worlds degrees higher than the ones of  any $\alpha$-worlds, it considers the $\neg\alpha$-worlds as plausible as they can be, while still not believed to be possibly the case ($*^5$ can be seen as a combination of $*^3$ and $*^0$). The reader should consider some examples involving functions $*^1-*^5$.
    
    \begin{proposition}
        For any $*^i$,  if $\alpha=\square^x\alpha_1$ or $\alpha=\square^\mathbb{N}\alpha_1$, $\mathcal{M},w\vDash [*^i\beta]\alpha \implies \mathcal{M}\vDash [*^i\beta]\alpha$ and, if $\alpha \in F_L$, $\mathcal{M},w\vDash [*^i\beta]\alpha \iff \mathcal{M},w\vDash\alpha$.
    \end{proposition}

        \textit{Proof: }For the first statement, assume that $\alpha=\square^x\alpha_1$ or $\alpha=\square^\mathbb{N}\alpha_1$ and that $\mathcal{M},w\vDash [*^i\beta]\alpha$, where $\mathcal{M}=(W,\pi, V)$. Therefore, $\mathcal{M}^i_\beta,w\vDash\alpha$. But since $\mathcal{M}^i_\beta$ is a $P$-model, it implies that $\mathcal{M}^i_\beta,w'\vDash\alpha$ for each world $w'\in W$. From this, it follows that $\mathcal{M},w'\vDash [*^i\beta]\alpha$ for each world $w'\in W$, that is, $\mathcal{M}\vDash [*^i\beta]\alpha$ (this is allowed because functions $*^i$ do not change the set $W$ of worlds). 

        For the second statement, $\alpha\in F_L$. Therefore, its truth conditions, as seen in Definition \ref{PSAT}, depend just on $W$ and $V$ and, therefore, are the same for $\mathcal{M}$ and $\mathcal{M}^i_\beta$. This concludes the proof. $\blacksquare$
    
    There's not very much to prove about arbitrary $*^i$ functions. This method of defining revision functions is very general and allows for modeling a wide range of interpretations of "receiving evidence". For instance, in general $[*^i\alpha_1][*^i\alpha_2]\alpha$ is not equivalent to $[*^i(\alpha_1\land \alpha_2)]\alpha$, even when $\alpha_1$ and $\alpha_2$ are compatible formulas of $F_L$.\footnote{i.e. when $\{\alpha_1,\alpha_2\}\nvdash_L\bot$.} Higher-order revisions about revisions, as well as revisions about beliefs, may also allow some strange results. An interesting example would be Moorean sentences, that is, sentences of the form $\alpha\land\neg\square^0\alpha$. When revised, those statements could never output a model in which they are believed, since $\square^0(\alpha\land\neg\square^0\alpha)$ is logically false in $P$. My interest in this paper, however, is analyzing the functions through AGM Postulates \ref{K*1}-\ref{K*8}. Those abnormal cases will not be addressed. 

\section{AGM for $P*$}\label{SAGMP*}
    The bases for the study of AGM within $P*$ are the following definitions I have formalized based on some informal comments in \cite{prole} and on philosophical \textit{desiderata} I consider to be faithful to the intended purpose of the AGM Postulates \ref{K*1}-\ref{K*8}.

    \begin{definition}\label{BSET}
        The belief set of a $P$-model $\mathcal{M}=(W,\pi,V)$ (such that $0\in Im(\pi)$) is the set $M_\mathcal M=\{\alpha\in F_{P*}:\mathcal{M}\vDash\square^0\alpha\}$
    \end{definition}

    \begin{definition}\label{RBSET}
        For each model $\mathcal{M}=(W,\pi,V)$, function $*^i$ and formula $\alpha\in F_L$ (such that $\mathcal{M}\nvDash\square^\mathbb{N}\neg \alpha$), the revised belief set of $\mathcal M$ is $(M_\mathcal{M})^i_\alpha=M_{\mathcal{M}^i_\alpha}$ (or, equivalently, $(M_\mathcal{M})^i_\alpha=\{\beta\in F_{P*}:\mathcal{M}\vDash[*^i\alpha]\square^0\beta\}$)
    \end{definition}

    From now on the subscript `$\mathcal{M}$' on `$M_\mathcal M$' will be omitted. Before stating the main definition proposed in this text, recall Notation \ref{DEF1}. From now on, the generalized version for $P*$ instead of $P$ (which define the function $Cn_{P*}(\cdot)$ and the relations $\vDash_{P*}$ and $\equiv_{P*}$) will also be used.\footnote{For $\Gamma\cup\{\alpha,\beta\}\subseteq F_{P*}$, both $\Gamma\vDash_{P*}\alpha$ and $\alpha\in Cn_{P*}(\Gamma)$ mean that: for all $P$-models $\mathcal M$, each of it's worlds $w$ is such that if $\mathcal M,w\vDash\gamma$ for all $\gamma\in \Gamma$, $\mathcal M,w\vDash\alpha$. $\alpha\equiv_{P*}\beta$ means $\alpha\vDash_{P*}\beta$ and $\beta\vDash_{P*}\alpha$.}

    \begin{definition}\label{SAT}
        A function $*^i$ of the logic $P*$ is said to satisfy an AGM Postulate when it satisfies the following analogous version of them for any model $\mathcal{M}$, any formulas $\alpha, \beta\in F_L$ and $\gamma\in F_{P*}$:
        \begin{equation}\label{M*1}
            Cn_{P*}(M^i_\alpha)=M^i_\alpha
        \end{equation}
        \begin{equation}\label{M*2}
            \alpha\in M^i_\alpha
        \end{equation}
        \begin{equation}\label{M*3}
        \begin{split}
            \text{If }\gamma\in M^i_{\alpha} \text{, then (for all } w\in W\text{) }\pi(w)=0\text{ and } \mathcal{M},w\vDash\alpha \\\text{imply } \mathcal{M}^i_{\alpha},w\vDash\gamma
        \end{split}
        \end{equation}
        \begin{equation}\label{M*4}
            \begin{split}\text{If }\neg \alpha \notin M \text{ and (for all } w\in W\text{) }\pi(w)=0\text{ and } \mathcal{M},w\vDash\alpha \\\text{imply } \mathcal{M}^i_{\alpha},w\vDash\gamma\text{, then } \gamma\in M^i_{\alpha}
            \end{split}
        \end{equation}
        \begin{equation}\label{M*6}
            \alpha \equiv_L\beta \implies M^i_\alpha=M^i_\beta
        \end{equation}
        \begin{equation}\label{M*7}
        \begin{split}
            \text{If }\gamma\in M^i_{\alpha\land\beta} \text{, then (for all } w\in W\text{) }\pi^i_\alpha(w)=0\text{ and } \mathcal{M},w\vDash\beta \\\text{imply } \mathcal{M}^i_{\alpha\land\beta},w\vDash\gamma
        \end{split}
        \end{equation}
        \begin{equation}\label{M*8}
            \begin{split}\text{If }\neg \beta \notin M^i_\alpha \text{ and (for all } w\in W\text{) }\pi^i_\alpha(w)=0\text{ and } \mathcal{M},w\vDash\beta \\\text{imply } \mathcal{M}^i_{\alpha\land\beta},w\vDash\gamma\text{, then } \gamma\in M^i_{\alpha\land\beta}
            \end{split}
        \end{equation}
    \end{definition}

    The restriction in Definition \ref{BSET} is established because, on models with no worlds of degree 0, every formula is trivially believed. But belief revision is executed only on consistent theories (recall Definition \ref{AGM*}). The restriction that $\alpha$ must not be knowingly false in Definition \ref{RBSET} may also be questioned. It exists because the intended use of implausible worlds is to model complete, eternal, and immutable certainty. Yet, AGM revisions are formalizations of agents that receive sufficient evidence for believing in the revised formula, and one must never come to believe in something known to be false. Formally, the absence of this restriction would imply that some revisions said to be successful\footnote{Or, in van Ditmarsch's terminology, propositionally successful.} in \cite{prole} would not satisfy Definition \ref{SAT}(\ref{M*2}) (known falsehoods would not be accepted even after revisions by them). Other revisions would trivialize, even when revised formulas are not logical falsities (revisions by known falsehoods would also output trivial theories). The restriction, therefore, is important for adequacy to both van Ditmarsch's and AGM's intuitions. 

    Notice that `$\alpha\in F_L$' in Definition \ref{BSET} is also a restriction. It is formally necessary for avoiding paradoxes, since, without it, no $*^i$ would be successful. The Moorean paradox $p_0\land\neg\square^0p_0$ can, in principle, be used in revision operators, but $\square^0(p\land\neg\square^0p)$ is false in any model. Without this restriction, no function would satisfy Definition  \ref{SAT}(\ref{M*2}). Furthermore, the purpose of this paper is not to include introspective revisions in the scope of "revisable" formulas. The main goal is, instead, to include the quantitative aspects of beliefs. Still, this approach to belief revision naturally allows the presence of formulas in $F_{P*}-F_L$ as elements of belief sets. This means that, despite the "revisable" beliefs being only the first-order ones, and never the introspective ones, the formulas that are \textit{believed} and \textit{discovered through revision} do include higher-order beliefs. Despite the fact that the agent never receives evidence about her own epistemic state, introspection is formalized, guaranteeing that agents know their own epistemic states. 

    Definition \ref{SAT} demands some adaptations with respect to the usual AGM postulates. The main one being the somewhat complex nature of  Definition \ref{SAT}(\ref{M*3},\ref{M*4}) (and the corresponding generalizations \ref{M*7},\ref{M*8}) when compared to Postulates \ref{K*3}-\ref{K*4} (as well as Postulates \ref{K*7}-\ref{K*8}). Despite the appearance,  Definition \ref{SAT}(\ref{M*3},\ref{M*4},\ref{M*7},\ref{M*8}), do, in fact, capture the same intuitive meaning of the original Postulates \ref{K*3}-\ref{K*4} and \ref{K*7}-\ref{K*8}, but know generalized for dealing with all the formulas $\gamma\in F_{P*}$, instead of only the ones in $F_L$. The main difference between, for instance, Postulate \ref{K*3} and  Definition \ref{SAT}(\ref{M*3}), is that $\gamma\in Cn(T\cup\{\alpha\})$ was transported to the modal setting as the expression $\forall w((\pi(w)=0\text{ and } \mathcal{M},w\vDash\alpha)\text{ imply } \mathcal{M}^i_{\alpha},w\vDash\gamma)$. This last expression states that, after revising $\alpha$, $\gamma$ becomes true in all the most plausible worlds that validate $\alpha$. Consequently, $\gamma$ comes to be true after the occurrence of any successful revision by $\alpha$. This means that, in the context of a belief revision realized by an agent, $\gamma$ is entailed by $\alpha$. This seems to be the adequate generalization of the notion of entailment stated by the expression $\gamma\in Cn(T\cup\{\alpha\})$. This is the case because, with this notion, the belief revision functions $*^i$ that satisfy the adapted versions of Postulates \ref{K*3}-\ref{K*4} and \ref{K*7}-\ref{K*8} are exactly the ones said to do so in \cite{prole}. Other adaptations of $\gamma\in Cn(T\cup\{\alpha\})$, that may be more intuitive (such as the simple $M\cup\{\alpha\}\vDash_{P*}\gamma$), would be too narrow, and wouldn't validate the informal statements made in \cite{prole}. The same translation of $\gamma\in Cn(T\cup\{\alpha\})$ to $\forall w((\pi(w)=0\text{ and }\mathcal{M},w\vDash\alpha)\text{ imply } \mathcal{M}^i_{\alpha},w\vDash\gamma)$ happens in Definition \ref{SAT}(\ref{M*4}), and explains it's unexpected complex structure.  Definition \ref{SAT}(\ref{M*7},\ref{M*8}) are just the foreseeable generalizations of  Definition \ref{SAT}(\ref{M*3},\ref{M*4}). 
    
     Notice that Definition \ref{SAT} lacks an equivalent of Postulate \ref{K*5}. This postulate is inapplicable given the restriction in Definition \ref{RBSET} that $\alpha$ must not be knowingly false for $M^i_\alpha$ to exist. All logical falsities in $L$ are knowingly false (being false in all worlds implies being false in $Plaus$). Therefore, $M^i_\alpha$ does not exist when $\vdash_L\neg\alpha$. Still, the restriction in Definition \ref{RBSET} guarantees that each $*^i$ preserves the intuition behind Postulate \ref{K*5}, as shown in Proposition \ref{PROP}. That is because the purpose of outputting the trivial theory after revising by a contradiction is outputting nonsense: a set unmanageable by AGM-like belief revision functions. The only difference is that, in this $P*$-based formalization, such unmanageable theories are never reached in the first place.

   Definitions \ref{BSET}-\ref{SAT}, by encoding similar intuitive meanings as the traditional AGM postulates, allow us to evaluate how well each function can formalize an AGM-like belief revision. The rest of this paper is devoted for that. 

    \begin{proposition}
        $*^1$ and $*^3$ do not satisfy Definition \ref{SAT}.
    \end{proposition}

    \textit{Proof: }The following counterexample shows that  $*^1$ and $*^3$ do not satisfy Definition \ref{SAT}(\ref{M*2}): As before, I'll draw attention only for the fragment of $F_{P*}$ with the two propositional variables $p_0\text{ and }p_1$. Consider the model $\mathcal{M}=(W,\pi,V)$ in which, again, $w_1\vDash p_0\land p_1$, $w_2\vDash p_0\land \neg p_1$, $w_3\vDash\neg p_0\land p_1$ and $w_4\vDash\neg p_0\land\neg p_1$. Moreover, $\pi(w_1)=\pi(w_2)=0$, $\pi(w_3)=1$ and $w_4\notin Plaus$. In $\mathcal{M}^1_{\neg p_0}$, $\pi^1_{\neg p_0}(w_1)=\pi^1_{\neg p_0}(w_2)=\pi^1_{\neg p_0}(w_3)=0$. Notice, now, that the same happens in $\mathcal{M}^3_{\neg p_0}$ (i.e., $\pi^3_{\neg p_0}(w_1)=\pi^3_{\neg p_0}(w_2)=\pi^3_{\neg p_0}(w_3)=0$). Therefore, $\mathcal{M}^1_{\neg p_0}\nvDash\square^0\neg p_0$ and $\mathcal{M}^3_{\neg p_0}\nvDash\square^0\neg p_0$, that is, $\neg p_0\notin M^1_{\neg p_0}$ and $\neg p_0\notin M^3_{\neg p_0}$. $\blacksquare$

     Notice that, although neither $*^1$ or $*^3$ induce an AGM-like belief revision function, they do, indeed, represent processes in which an agent comes to consider a formula more plausible in some way. In this sense, then, each of them captures a notion of belief revision. Not being an AGM-like belief revision does not mean not being useful or valuable for belief dynamics.

    \begin{theorem}\label{T16}
    All $*^i$ functions presented satisfy Definition \ref{SAT}(\ref{M*1},\ref{M*6}).
    \end{theorem}    

    \textit{Proof: } For Definition \ref{SAT}(\ref{M*1}), let $\mathcal{M}=(W,\pi,V)$ be a $P$-model and $\alpha\in F_L$. $\mathcal{M}^i_\alpha$ is also a $P$-model. Assume now that $M^i_\alpha$ is well-defined (i.e., that some $w\in W$ is such that $\pi^i_\alpha(w)=0$) and $M^i_\alpha\vDash_{P*}\beta$. Consider now the triple $\mathcal M^i_\alpha|_0=(W_0,\pi_0,V_0)$ where $W_0=\{w\in W:\pi^i_\alpha(w)=0\}$, $\pi_0:W_0\to\mathbb N$ is such that $ \pi_0(w)=\pi^i_\alpha(w)$ and $V_0(p_i)=V(p_i)\cap W_0$. From the fact that $M^i_\alpha$ is well-defined, it is easy to check that $\mathcal M^i_\alpha|_0$ is a $P$-model (Definition \ref{DEF PMODEL}). Notice, also, that, by Definition \ref{RBSET}, $\gamma\in M^i_\alpha \iff \mathcal M^i_\alpha|_0\vDash\gamma$. Therefore, by the fact that $M^i_\alpha\vDash_{P*}\beta$, $\mathcal{M}^i_\alpha|_0\vDash \beta$. Finally, $\beta\in M^i_\alpha$. This proofs that $Cn_{P*}(M^i_\alpha)\subseteq M^i_\alpha$. The statement $M^i_\alpha\subseteq Cn_{P*}(M^i_\alpha)$ follows from the Tarskian nature of $Cn$.

        For Definition\ref{SAT}(\ref{M*6}), let $\mathcal{M}=(W,\pi,V)$ be a $P$-model and $\alpha, \beta\in F_L$ be such that $\alpha\equiv_{L}\beta$. Therefore, $\alpha\equiv_{P*}\beta$ (Definition \ref{PSAT} and Notation \ref{DEF1} imply that $P*$ preserves the consequence relation of $L$). This means that the worlds that satisfy $\alpha$ are the same as the ones that satisfy $\beta$ in $\mathcal{M}$. Hence, by Definition \ref{*}, $\pi^i_\alpha=\pi^i_\beta$. Therefore, $\mathcal{M}^i_\alpha=\mathcal{M}^i_\beta$ and $M^i_\alpha=M^i_\beta$. $\blacksquare$

    \begin{proposition}\label{PROP}
            For all $*^i$ functions presented, formula $\alpha\in F_L$ and $P$-model $\mathcal{M}$,  $M^i_\alpha$ (when it exists) is different from $F_{P*}$.
        \end{proposition}

        \textit{Proof: }Take an arbitrary $\mathcal{M}=(W,\pi,V)$ such that there is a world of $\pi$-degree $0$ and an arbitrary $\alpha\in F_L$ such that some world in $Plaus$ validates $\alpha$.\footnote{Notice that, despite the assumptions, $(\mathcal M,\alpha)$ is an arbitrary pair that defines a belief set $M^i_\alpha$ (check Definitions \ref{BSET} and \ref{RBSET}).} From the last fact, $Dom(\overline{\pi}^i_\alpha)\neq \varnothing$ (by Definition \ref{*}). By the well-ordering principle for $\mathbb{N}$, there is a world $w$ such that $\overline\pi^i_\alpha(w)=\min\{\overline{\pi}^i_\alpha(w):w\in W\}$. By Definition \ref{*}(\ref{normalization}), $\pi^i_\alpha(w)=0$. Assume, for arriving at a contradiction, that $M^i_\alpha=F_{P*}$. Therefore $\mathcal{M}^i_\alpha\vDash \square^0\bot$. Since no world validates $\bot$, no world in $W$ has degree $0$ in $\pi^i_\alpha$. Absurd by the fact that $\pi^i_\alpha(w)=0$. $\blacksquare$

        This shows that, despite the lack of an equivalent of Postulate \ref{K*5}, the philosophical intuition behind it is preserved. Now, the main theorem of this section.

    \begin{theorem}
        Functions $*^2$,$*^4$, $*^0$ and $*^5$ satisfy Definition \ref{SAT}. 
    \end{theorem}

    \textit{Proof: }This proof consists of the conjunction of several lemmas. Since  Definition \ref{SAT}(\ref{M*3},\ref{M*4}) are instances of Definition \ref{SAT}(\ref{M*7},\ref{M*8}) for $\alpha=\top$, it suffices to show that the functions satisfy the last two along with Definition \ref{SAT}(\ref{M*2}). Theorem \ref{T16} deals with the remaining items.

        \begin{lemma}
            $*^2$, $*^4$, $*^0$ and $*^5$ satisfy Definition \ref{SAT}(\ref{M*2}).
        \end{lemma}

        \textit{Proof: }For $*^2$, let $\alpha \in F_L$ and $\mathcal{M}=(W,\pi, V)$ be both arbitrary given the proper restrictions for the existence of $M^2_\alpha$. Proving that $\mathcal{M}^2_\alpha\vDash\square^0\alpha$ is enough to show that $\alpha\in M^2_\alpha$. Take an arbitrary world $w$ such that $\pi^2_\alpha(w)=0$. Since normalization does not change the normalized function's domain, $w\in Dom(\overline{\pi}^2_\alpha)$. As a result, by Definition \ref{*}(\ref{*2}) $\mathcal{M},w\vDash \alpha$. But $\alpha\in F_L$, so changes on $\pi$ will not change the worlds that satisfy it. Therefore, $\mathcal{M}^2_\alpha, w\vDash \alpha$. Since $w$ was an arbitrary world of $\pi^2_\alpha$-degree $0$,  $\mathcal{M}^2_\alpha\vDash\square^0\alpha$.

            For $*^4$, in a similar way, consider $\alpha$, $\mathcal{M}$ and $w$ such that $\pi^4_\alpha(w)=0$. I will show that $\mathcal{M}, w\vDash \alpha$ and, as in the paragraph above, it proves the lemma for $*^4$. By the fact that $\pi^4_\alpha(w)=0$, $\overline{\pi}^4_\alpha(w)=\min\{\overline{\pi}^4_\alpha(w): w\in W\}$. Assume, now, for the sake of implying a contradiction, that $\mathcal M,w\nvDash\alpha$. By Definition \ref{*}(\ref{*4}),  $\overline{\pi}^4_\alpha(w)>\max\{\pi(w):\mathcal M,w\vDash\alpha\}$. But since $M^4_\alpha$ is well-defined $\{\pi(w):\mathcal M,w\vDash\alpha\}\neq\varnothing$. Therefore, $\overline{\pi}^4_\alpha(w)>\overline{\pi}^4_\alpha(w')$ for some $w'\in W$, which contradicts the previously stated fact that $\overline{\pi}^4_\alpha(w)=\min\{\overline{\pi}^4_\alpha(w): w\in W\}$. From this, $\mathcal M,w\vDash\alpha$. The argument for $*^0$ is omitted for being too similar to the one for $*^4$ (one can simply replace `$4$' for `$0$' in this paragraph).
            
            For $*^5$ the proof is done by contrapositive. Let $\alpha\in F_L$ and $\mathcal M=(W,\pi,V)$ be such that $\mathcal M\nvDash \square^\mathbb N \neg\alpha$ and take any world $w$ such that $\mathcal{M}^5_\alpha,w\nvDash\alpha$, or, equivalently, $\mathcal{M},w\nvDash\alpha$. It suffices to show that $\pi^5_\alpha(w)\neq 0$. Since $\mathcal{M},w\nvDash\alpha$, $\overline\pi^5_\alpha(w)\geq 1$ by Definition \ref{*}(\ref{*5}). But since $\mathcal M\nvDash \square^\mathbb N \neg\alpha$, there is a world in $Plaus$ that validates it. Let $w_{min}$ be a world with minimum $\pi$-degree which validates $\alpha$ (it exists by the well-order of $\mathbb N$). Definition \ref{*}(\ref{*5}) states that $\overline\pi^5_\alpha(w_{min})=0$. Therefore, $\overline\pi^5_\alpha(w_{min})<\overline\pi^5_\alpha(w)$ and $\overline\pi^5_\alpha(w)\neq \min\{\overline\pi^5_\alpha(w'):w'\in W\}$. Definition \ref{*}(\ref{normalization}) provides the intended result that $\pi^5_\alpha(w)\neq 0$. $\blacksquare$

        \begin{lemma}\label{L7}
            $*^2$, $*^4$, $*^0$ and $*^5$ satisfy Definition \ref{SAT}(\ref{M*7}).
        \end{lemma}

        \textit{Proof: } I'll begin by prooving for $*^2$, $*^4$ and $*^0$. Take an arbitrary model $\mathcal{M}=(W,\pi,V)$ along with arbitrary formulas $\alpha, \beta\in F_L$ and $\gamma\in F_{P*}$. Assume that, for $i\in\{0,2,4\}$, $M^i_{\alpha\land\beta}$ is well-defined and $\mathcal{M}^i_{\alpha\land\beta}\vDash\square^0\gamma$ (i.e., $\gamma\in M^i_{\alpha\land\beta}$). Consider, now, a world $w\in W$ such that $\pi^i_\alpha (w)=0$ and $\mathcal{M},w\vDash\beta$. I will show that $\mathcal{M}^i_{\alpha\land\beta}, w\vDash\gamma$. 

        We know, from the fact that $M^i_{\alpha\land\beta }$ is well-defined, that some $w\in Plaus$ validates $\alpha\land\beta$. Therefore, $Dom(\overline\pi^i_{\alpha\land\beta})\neq\varnothing$ and there is a $w'\in W$ such that $\overline\pi^i_{\alpha\land\beta}(w')=\min\{\overline\pi^i_{\alpha\land\beta}(w'') : w''\in W\}$. Notice that $\pi^i_{\alpha\land\beta}(w')=0$, which means that proving $\overline\pi^i_{\alpha\land\beta}(w')= \overline \pi^i_{\alpha\land\beta}(w)$ suffices to show that $\pi^i_{\alpha\land\beta}(w)=0$. Since $*^0$, $*^2$ and $*^4$ are all successful and $\pi^i_\alpha (w)=0$, $\mathcal M,w\vDash\alpha$. Consequently, $\mathcal M,w\vDash\alpha\land\beta$, which guarantees the existence of $\overline \pi^i_{\alpha\land\beta}(w)$. From $\overline\pi^i_{\alpha\land\beta}(w')=\min\{\overline\pi^i_{\alpha\land\beta}(w'') : w''\in W\}$, it follows that $\overline\pi^i_{\alpha\land\beta}(w')\leq \overline \pi^i_{\alpha\land\beta}(w)$. From $\mathcal M,w\vDash\alpha$ and $\mathcal M,w\vDash\alpha\land\beta$, $\overline \pi^i_{\alpha\land\beta}(w)=\pi(w)=\overline \pi^i_{\alpha}(w)$ (By Definition \ref{*}(\ref{*2},\ref{*4},\ref{*0})). Also, since $\pi^i_\alpha (w)=0$, $\overline\pi^i_\alpha (w)=\min\{\overline\pi^i_\alpha (w''):w''\in W\}$. Therefore, $\overline\pi^i_\alpha (w)\leq \overline\pi^i_\alpha(w')$.  By Definition \ref{*}(\ref{*2},\ref{*4},\ref{*0}), each world was a $\overline{\pi}^i_\alpha$-degree less or equal to its $\overline{\pi}^i_{\alpha\land\beta}$-degree (the revision by $\alpha\land\beta$ increases the degree of more worlds). Hence, $\overline\pi^i_\alpha(w')\leq \overline\pi^i_{\alpha\land\beta}(w')$. Putting everything together, $\overline \pi^i_{\alpha\land\beta}(w)=\overline \pi^i_{\alpha}(w)\leq \overline\pi^i_\alpha(w')\leq \overline\pi^i_{\alpha\land\beta}(w')$. This shows that $\overline\pi^i_{\alpha\land\beta}(w')= \overline \pi^i_{\alpha\land\beta}(w)$. Hence, $\pi^i_{\alpha\land\beta}(w)=0$. Finally, the hypothesis that $\mathcal{M}^i_{\alpha\land\beta}\vDash\square^0\gamma$ implies $\mathcal{M}^i_{\alpha\land\beta}, w\vDash\gamma$.

            $*^5$ demands its own proof. Let $\alpha,\beta\in F_L$ and $\mathcal M=(W,\pi,V)$ be such that $M^i_{\alpha\land\beta}$ is well-defined. Take an arbitrary $w\in W$ such that $\mathcal M,w\vDash \beta$ and $\pi^5_\alpha(w)=0$. Now, I'll to show that $\mathcal M^5_{\alpha\land\beta},w\vDash \gamma $ for any $\gamma\in F_{P*}$ such that $\mathcal M^5_{\alpha\land\beta}\vDash\square^0\gamma$. This will be done by showing that $\pi^5_{\alpha\land\beta}(w)=0$. Since $\pi^5_\alpha(w)=0$, $\overline\pi^5_\alpha(w)=\min\{\overline\pi^5_\alpha(w'):w'\in W\}$. But since $\mathcal M \nvDash\square^\mathbb N\neg(\alpha\land\beta)$,\footnote{$M^i_{\alpha\land\beta}$ is well-defined.}  $\mathcal M \nvDash\square^\mathbb N\neg\alpha$. From this it follows, by Definition \ref{*}(\ref{*5}), that $\min\{\overline\pi^5_\alpha(w'):w'\in W\}=0$ (the $\overline\pi^5_\alpha$-value of the world with $\pi$-minimum degree among those that validate $\alpha$ must be 0). Therefore, $\overline\pi^5_{\alpha}(w)=0$ and $\pi(w)=\min\{\pi(w'):\mathcal M,w\vDash\alpha\}$. But, by the hypothesis that $\mathcal M,w\vDash \beta$, it follows that $\pi(w)=\min\{\pi(w'):\mathcal M,w\vDash\alpha\land\beta\}$. Finally, by Definition \ref{*}(\ref{*5}), $\pi^5_{\alpha\land\beta}(w)=0$. $\blacksquare$

        \begin{lemma}\label{L8}
            $*^2$, $*^4$, $*^0$ and $*^5$ satisfy Postulate \ref{M*8}.
        \end{lemma}

        \textit{Proof: } For each $*^i\in \{*^2,*^4,*^0,*^5\}$, assume that $\mathcal M^i_\alpha\nvDash\square^0\neg\beta$ and both $M^i_\alpha$ and $M^i_{\alpha\land\beta}$ are well-defined. I'll begin with a proof that all worlds in $W$ with $\pi^i_{\alpha\land\beta}$-degree equal to $0$ must both (i) validate $\beta$ and (ii) have $\pi^i_\alpha$-degree 0.
        
        Consider an arbitrary $w\in W$ such that $\pi^i_{\alpha\land\beta}(w)=0$. By the postulate of success, valid for each of the functions, $\mathcal{M},w\vDash\alpha\land\beta$. Therefore, (i) is proved. Since $\pi^i_{\alpha\land\beta}(w)=0$, $\overline{\pi}^i_{\alpha\land\beta}(w)= \min\{\overline{\pi}^i_{\alpha\land\beta}(w'):w'\in W\}$. By the postulate of success, $\min\{\overline{\pi}^i_{\alpha\land\beta}(w'):w'\in W\}=\min\{\overline{\pi}^i_{\alpha\land\beta}(w'):\mathcal M,w'\vDash \alpha\land\beta\}$. Furthermore, $\pi(w)= \min\{\pi(w'):\mathcal{M},w'\vDash\alpha\land\beta\}$. In the case when $i\in\{0,2,4\}$,  this is true because, for any $w_1\in W$, $\mathcal{M},w_1\vDash\alpha\land\beta \implies\overline{\pi}^i_{\alpha\land\beta}(w_1)=\pi(w_1)$, but, as I have pointed out, $\overline{\pi}^i_{\alpha\land\beta}(w)= \min\{\overline{\pi}^i_{\alpha\land\beta}(w'):\mathcal M,w'\vDash \alpha\land\beta\}$. In the case of $i=5$, we can easily check in Definition \ref{*}(\ref{*5}) that $\pi(w)=\min\{\pi(w'):\mathcal M,w'\vDash \alpha\land\beta\}$ when $\pi^5_{\alpha\land\beta}(w)=0$.
            
            From $\mathcal M^i_\alpha\nvDash\square^0\neg\beta$ we know there is a $w'$ such that $\pi^i_\alpha(w')=0$ and $w'\vDash \beta$ (and, by the postulate of success, $\mathcal{M}, w'\vDash\alpha\land\beta$). From the fact that $\pi^i_\alpha(w')=0$, we get that $\pi(w')= \min\{\pi(w'):\mathcal{M}, w'\vDash \alpha\}$. Therefore, \[\pi(w')=\min\{\pi(w'):\mathcal{M}, w'\vDash \alpha\}\leq \min\{\pi(w):\mathcal{M},w'\vDash\alpha\land\beta\} = \pi(w)\]On the other hand, $\mathcal{M}, w'\vDash \alpha\land\beta$, so \[\pi(w)= \min\{\pi(w):\mathcal{M},w'\vDash\alpha\land\beta\}\leq\pi(w')\]
            
            Antisymmetry states that $\pi(w)=\pi(w')$ and, therefore, (since both validate $\alpha$) $\pi^i_\alpha(w)=\pi^i_\alpha(w')=0$. This proves that, for  all $w\in W$, $\pi^i_{\alpha\land\beta}(w)=0\implies \pi^i_\alpha(w)=0\text{ and } \mathcal M,w\vDash\beta$. 
            
            The last part of the proof begins with the assumption that, for an arbitrary $\gamma$, all $w\in W$ are such that $\pi^i_\alpha(w)=0$ and $ \mathcal{M},w\vDash\beta$ imply $\mathcal{M}^i_{\alpha\land\beta},w\vDash\gamma$. Consider now a world $w$ with $\pi^i_{\alpha\land\beta}(w)=0$. The paragraphs above show that $\pi^i_\alpha(w)=0$ and $\mathcal{M},w\vDash\beta$. Therefore, $\mathcal{M}^i_{\alpha\land\beta},w\vDash\gamma$. Hence, $\mathcal M^i_{\alpha\land\beta}\vDash\square^0\gamma$ and $\gamma\in M^i_{\alpha\land\beta}$. This concludes the proof. $\blacksquare$

    This was the last lemma, which means the theorem is proved. $\blacksquare$

It is possible, now, to consider all four of those functions to be AGM-like belief revision functions, at least in a certain sense. This is quite surprising, as they're all considerably different from each other. This is because an AGM-like function of $P*$ may not yet be a good representation of AGM belief revisions. Since $P*$ is more informative than the basic AGM theory, the AGM postulates are not enough to capture the hole intuition behind AGM revisions \textit{for this specific system}. This is not a defect of the postulates' adaptations. They do, in fact, capture the same intuitions behind the original ones, but this is not enough due to the gain in expressiveness. In this direction, the next section draws distinctions between $*^2$, $*^4$, $*^5$ and $*^0$. In the end, I argue that $*^0$, is (among those) the one that best captures the notion AGM-like belief revision functions intend to express.

\section{$*^0$, $*^2$, $*^4$ and $*^5$ Beyond AGM Postulates}\label{SAGMP*+}

    I begin by showing a result related to $*^2$ that is quite undesirable when the function is seen as an AGM-like belief revision function.

    \begin{proposition}\label{Paradox *2}
        For each $P$-model $\mathcal{M}$ and $\alpha\in F_L$, there is no $(M^2_\alpha)^2_{\neg\alpha}$
    \end{proposition}

    \textit{Proof: }Consider a model $\mathcal{M}^2_\alpha$ for arbitrary $\mathcal{M}=(W,\pi,V)$ and $\alpha$. For each $w\in Dom(\pi^2_\alpha)$, $\mathcal{M}, w\vDash \alpha$ since, otherwise, by Definition \ref{*}(\ref{*2}), $w\notin Dom(\overline{\pi}^2_\alpha)=Dom(\pi^2_\alpha)$. Therefore, $\mathcal{M}^2_\alpha\vDash \square^\mathbb{N}\alpha$, i. e. $\mathcal{M}^2_\alpha\vDash \square^\mathbb{N}\neg \neg \alpha$. By Definition \ref{RBSET}, revisions by knowingly false formulas do not induce a revised belief set. $\blacksquare$

    Notice that, even if we got rid of the restriction in Definition \ref{RBSET}, $*^2$ would give us strange results related to $(\mathcal{M}^2_\alpha)^2_{\neg\alpha}$. $Dom((\pi^2_\alpha)^2_{\neg\alpha})=\varnothing$ and, therefore $(\mathcal{M}^2_\alpha)^2_{\neg\alpha}\vDash \square^\mathbb{N}\bot$. This should not be the case for an AGM belief revision function. Ideally, those functions formalize agents embedded in a dynamic setting with evidence and information that may change over time. It is perfectly rational, for those agents, to accept a formula and, after changes in the state of affairs, accept its negation. The intended way to formalize this process within AGM would be two consecutive revisions, one by a formula and another by its negation. Unfortunately, $*^2$ states that such agents would be irrational beasts who believe in absurdities. This result is not unexpected when one thinks about the relation between the philosophical interpretation of the AGM theory and of $*^2$. As said in the Introduction, AGM revisions assume that the agent is holding beliefs about a changing environment, but $*^2$ was interpreted, in Section \ref{SP*}, as the process of comprehending a formal proof of a formula, and mathematical objects are static by nature. $*^2$ represents the acquisition of certain and eternal knowledge. Receiving a proof $\alpha$ and then one of $\neg\alpha$ should, indeed, lead the agent to absurdity, since it does not mean that the situation related to $\alpha$ changed, but that there is a true contradiction. 

   Here we have another paradox, now related to $*^4$. 

    \begin{proposition}
        Given two arbitrary $P$-models $\mathcal{M}=(W,\pi,V)$ and $\mathcal{M'}=(W',\pi',V')$ such that $W=W'$, $V=V'$ and $Plaus=Plaus'$ along with an $\alpha\in F_L$\footnote{$M'=M_{\mathcal M'}$.}
        \[(M^4_\alpha)^4_{\neg\alpha}=(M')^4_\alpha)^4_{\neg\alpha}\]
    \end{proposition}

    \textit{Proof: }First, I will prove that
        \begin{lemma}
        For any $\alpha\in F_L$ and $P$-model $\mathcal{M}=(W,\pi,V)$
            \[\min\{\\\overline{(\pi^4_\alpha)}\,^4_{\neg\alpha}(w): w\in W\} =\max\{\pi^4_\alpha(w): \mathcal{M},w\vDash\neg\alpha\}\]
        \end{lemma}
        \textit{Proof: } Notice that $\min\{\overline{(\pi^4_\alpha)}\,^4_{\neg\alpha}(w): w\in W\}= \min\{\pi^4_\alpha(w): \mathcal{M},w\vDash\neg\alpha\}$, since all worlds that do not satisfy $\neg\alpha$ are going to receive $\overline{(\pi^4_\alpha)}\,^4_{\neg\alpha}$-values greater than the ones that satisfy (Definition \ref{*}(\ref{*4})). But since all worlds that do not validate $\alpha$ are assigned the same value in $\pi^4_\alpha$, $\min\{\pi^4_\alpha(w): \mathcal{M},w\vDash\neg\alpha\}= \max\{\pi^4_\alpha(w): \mathcal{M},w\vDash\neg\alpha\}$. This proves the lemma. $\blacksquare$

        Now, take an arbitrary $w\in W$. Either $\mathcal{M},w\vDash\alpha$ or $\mathcal{M},w\nvDash\alpha$
        \begin{enumerate}
            \item Assume that $\mathcal{M},w\vDash\alpha$. We know, by Definition \ref{*}(\ref{*4}), that \[\overline{(\pi^4_\alpha)}\,^4_{\neg\alpha}(w)= \max\{\pi^4_\alpha(w):w\vDash \neg\alpha\}+1\] and, hence, \[(\pi^4_\alpha)^4_{\neg\alpha}(w)= \max\{\pi^4_\alpha(w):w\vDash \neg\alpha\}-\min\{\overline{(\pi^4_\alpha)}\,^4_{\neg\alpha}(w): w\in W\}+1\] The lemma guarantees, then, that $(\pi^4_\alpha)^4_{\neg\alpha}(w)=1$. The exact same argument applies to $\pi'$, which means that $(\pi'\,^4_\alpha)^4_{\neg\alpha}(w)=1= (\pi^4_\alpha)^4_{\neg\alpha}(w)$.
            
            \item Assume that $\mathcal{M},w\nvDash\alpha$. Since $\mathcal{M},w\vDash\neg\alpha$, $\overline{(\pi^4_\alpha)}\,^4_{\neg\alpha}(w)= \pi^4_\alpha(w)$ and \[(\pi^4_\alpha)^4_{\neg\alpha}(w)= \pi^4_\alpha(w)- \min\{\overline{(\pi^4_\alpha)}\,^4_{\neg\alpha}(w): w\in W\}\]Also, (by the previous lemma) \[\min\{\overline{(\pi^4_\alpha)}\,^4_{\neg\alpha}(w): w\in W\}= \max\{\pi^4_\alpha(w):\mathcal{M},w\vDash\neg\alpha\}\]But, by Definition \ref{*}(\ref{*4}), all worlds that do not satisfy $\alpha$ have the same $\pi^4_\alpha$-degree. Therefore, $\max\{\pi^4_\alpha(w):\mathcal{M},w\vDash\neg\alpha\}=\pi^4_\alpha(w)$ and \[(\pi^4_\alpha)^4_{\neg\alpha}(w)= \pi^4_\alpha(w)- \pi^4_\alpha(w)=0\] Again, everything equally holds for $\pi'$, so $(\pi'\,^4_\alpha)^4_{\neg\alpha}(w)=0= (\pi^4_\alpha)^4_{\neg\alpha}(w)$.
        \end{enumerate}
        Since, in both cases, $(\pi'\,^4_\alpha)^4_{\neg\alpha}(w)=(\pi^4_\alpha)^4_{\neg\alpha}(w)$, from the fact that $Plaus=Plaus'$, we have   $(\pi'\,^4_\alpha)^4_{\neg\alpha}=(\pi^4_\alpha)^4_{\neg\alpha}$. Therefore, $(\mathcal{M}^4_\alpha)^4_{\neg\alpha} =(\mathcal{M'}^4_\alpha)^4_{\neg\alpha}$. The belief sets of such models, then, are also equal.$\blacksquare$

    The result states that, in $*^4$, no matter how complex and different two preference functions with the same domain are, after receiving only two new pieces of information, their agents forget all their previous disagreements. This is not the intended result of an AGM-like revision function. AGM was made to formalize \textit{rational} belief change, which means that revising $\alpha$ and then $\neg\alpha$ should lead to a state similar to the one held after revising only $\neg\alpha$. On the other hand, $*^4$ radically changes most epistemic states after realizing this process by always outputting the same preference function.

    $*^5$ Does not, at least as far as I can see, present any paradoxes involving the sets $M^5_\alpha$. But there is, still, a big problem on the way $*^5$ manipulates the preference functions in order to arrive at the desired results. This can be seen by the following proposition:

    \begin{proposition}\label{Paradox *5}
        There are models $\mathcal{M}=(W,\pi,V)$ such that, for all $w\in W$ in which $\mathcal{M},w\vDash\alpha$, $\pi^5_\alpha(w)=\pi(w)$ while, for all $w'$ in which $\mathcal{M},w'\nvDash\alpha$, $\pi^5_\alpha(w')<\pi(w')$.
    \end{proposition}

    \textit{Proof: }Limmiting $F_L$ to the single propositional variable $p_0$, consider the case in wich $W=\{w,w'\}$, $w\vDash p_0$, $w'\nvDash p_0$, $\pi(w)=0$ and $\pi(w')=2$. It can be easly checked that $\pi^5_{p_0}(w)=0=\pi(w)$ and $\pi^5_{p_0}(w')=1<2$. $\blacksquare$

    The $\pi$-degree of $w'$ here could be any number greater than $1$. The proposition states that, in some cases, $*^5$-revising by a formula makes the formula less plausible relative to the other formulas \textit{in the perspective of the agent that realized the revision}. This is not the intended behavior of a formalization of belief revisions, since the main idea is to represent situations in which an agent receives some sort of compelling evidence for a belief. 
    
    The results in this section show that, since, with $P*$, van Ditmarsch generalized the notion of "belief" to include quantitative aspects, more postulates about belief revisions should be introduced to deal with the paradoxes just stated. It is not within the scope of this paper to discover what those postulates are, but, as the propositions in this section show, those postulates should be such that they are not satisfied by any of the AGM-like functions presented in \cite{prole}. The only function that truly seems a good contender for an AGM revision within $P*$ is $*^0$. None of the problems stated by Propositions \ref{Paradox *2}-\ref{Paradox *5} occur in $*^0$. By preserving the internal order of the worlds that satisfy and that do not satisfy the revised formula, $*^0$ is faithful to agents' previous epistemic states. After receiving information, agents preserve most of their preferences. In \cite{gar}, G\"arderfors constantly states how \textit{minimality} is one of the main desiderata of rational belief change. Agents should only change their beliefs as much as \textit{needed} to adapt themselves to new available information. Given that $*^0$ seems to be a nice formalization of AGM belief revisions, I have implemented the function on the language \textsf{C}. The code is available at \textcolor{blue}{\url{https://github.com/FelipeNunesCamargo/AGM-research}}) and it outputs the model $\mathcal{M}^0_\alpha$ when $\mathcal M$ and $\alpha$ are given by the user. Also, the code can be easily adapted for other revision functions definable in $P*$.

\section{Final Remarks}

This paper has proposed new tools to develop studies on single-agent quantitative DELs with preference functions. This was done by adapting AGM postulates for the DEL setting $P*$. Based on the exposition of Sections \ref{SAGM}-\ref{SP*}, this paper provided a new revision function that is a good fit for capturing the philosophical notion of a belief revision through the lens of the AGM postulates. This was done by pointing out the existence of paradoxes on several belief revision functions presented in \cite{prole}, as well as the apparent lack of such paradoxes for the proposed function.

The rest of the final remarks is devoted to future works. Some of those may involve the development of syntactic methods (Hilbert calculi, Natural Deduction, etc.) for dealing with $P$ (as well as some revision operators of $P*$). Modal tableau, similar to the ones in \cite{pri}, but with nodes including both a world and a model, strikes me as a viable decision procedure. Beyond that, the explicit formulation of the postulates that extend the adapted versions of AGM in order to avoid paradoxes such as the ones presented in Propositions \ref{Paradox *2}-\ref{Paradox *5} would be a great addition. The traditional AGM theory based on Postulates \ref{K*1}-\ref{K*8} has developed, as shown in \cite{han}, several equivalences and relations between dynamic processes. Revision is one of these processes. It seems possible to discover and prove adapted versions of at least some of those equivalences. Furthermore, the use of $\mathbb{N}$ as the codomain of preference functions is a somewhat arbitrary choice. To obtain more prolific applications of $P*$ one could replace $\mathbb{N}$ for $\mathbb{R}$ or even $\mathbb{R}^n$ and use the mathematical apparatus of Calculus and Analysis for defining complex revision operations. This could drastically increase the applied potential of $P*$ as well as its expressive power, since it would capture more complex quantitative epistemic notions. At last, one may consult \cite{prole} for more directions to future developments. 

\begin{credits}
\subsubsection{\ackname}  This paper gathers the main original results of an undergraduate research project carried out by the author at CLE-UNICAMP (Centre of Logic, Epistemology and the History of Science at State University of Campinas). The study was funded by FAPESP (Research Support Foundation of the State of São Paulo) (grant number 2024/07209-7). The advisor was Marcelo Esteban Coniglio, for whose guidance I am thankful. Furthermore, I thank the anonymous referee, who provided great suggestions and corrections.

\subsubsection{\discintname}
The authors have no competing interests to declare that are
relevant to the content of this article.
\end{credits}
%
%
%
%

\end{document}